\documentclass[twoside]{article}
\usepackage[accepted]{aistats2016}
\usepackage{times} 
\usepackage{subscript} 
\usepackage[geometry]{ifsym} 
\usepackage{color} 
\usepackage[utf8]{inputenc} 
\usepackage{amsmath,amssymb,amsthm} 
\usepackage{graphicx} 
\usepackage{subfigure} 
\usepackage{booktabs} 
\usepackage{multirow} 
\usepackage{multicol}
\usepackage[round]{natbib}
\usepackage{caption}
\usepackage{paralist}
\usepackage{bbm}
\usepackage{url}
\usepackage[lined, ruled, boxed, linesnumbered, nofillcomment]{algorithm2e}
\usepackage{titlesec} 

\newcommand{\Do}[1]{\text{do}(#1)}
\newcommand{\man}[1]{\text{man}(#1)}
\newcommand{\calI}{\mathcal{I}}
\newcommand{\calJ}{\mathcal{J}}

\newtheorem{theorem}{Theorem}

\newtheorem{definition}[theorem]{Definition}

\usepackage{titlesec} 
\titlespacing\section{0pt}{12pt plus 4pt minus 2pt}{0pt plus 2pt minus 2pt}
\titlespacing\subsection{0pt}{12pt plus 4pt minus 2pt}{0pt plus 2pt minus 2pt}
\titlespacing\subsubsection{0pt}{12pt plus 4pt minus 2pt}{0pt plus 2pt minus 2pt}

\begin{document}
\twocolumn[
\aistatstitle{Multi-Level Cause-Effect Systems}
\aistatsauthor{ Krzysztof Chalupka \And Pietro Perona \And Frederick Eberhardt}
\aistatsaddress{California Institute of Technology}
]

\begin{abstract}
We present a domain-general account of causation that applies to settings in which macro-level causal relations between two systems are of interest, but the relevant causal features are poorly understood and have to be aggregated from vast arrays of micro-measurements. Our approach generalizes that of \citet{Chalupka2015} to the setting in which the macro-level effect is not specified. 
We formalize the connection between micro- and macro-variables in such situations and provide a coherent framework describing causal relations at multiple levels of analysis. We present an algorithm that discovers macro-variable causes and effects from micro-level measurements obtained from an experiment. We further show how to design experiments to discover macro-variables from observational micro-variable data. Finally, we show that under specific conditions, one can identify multiple levels of causal structure. Throughout the article, we use a simulated neuroscience multi-unit recording experiment to illustrate the ideas and the algorithms. 
\end{abstract}

\section{INTRODUCTION}
In many scientific domains, detailed measurement is an indirect tool to construct and identify macro-level features of interest which are not yet fully understood. For example, climate science uses satellite images and radar data to understand large scale weather patterns. In neuroscience, brain scans or neural recordings constitute the basis for research into cognition. In medicine, body monitors and gene sequencing are used to predict macro-level states of the human body, such as health outcomes. In each case the aim is to use the micro-level data to discover what the relevant macro-level features are that drive, say, the ``El Ni{\~n}o'' weather pattern, face recognition or debilitating diseases. We propose a principled approach for the identification of macro-level causes and effects from high-dimensional micro-level measurements. Standard approaches to causation using graphical models \citep{Spirtes2000, Pearl2009} or potential outcomes \citep{Rubin1974} presuppose this step---these methods focus on discovering the causal relations among a \emph{given} set of well-defined causal variables. Our approach does not rely on domain experts to identify the causal relata but constructs them automatically from data.

Throughout the article we use the setting of a neuroscientific experiment with high-dimensional input stimuli (images), and high-dimensional output measurements (multi-unit recordings) to illustrate our approach. We emphasize, however, that our theoretical results are entirely domain-general. 

Our contribution is threefold:
\begin{compactenum}
\item We rigorously define how the constitutive relations (supervenience) between micro- and macro-variables combine with the causal relations among the macro-variables when both the macro-cause and macro-effect have not been pre-defined, but have to be constructed from micro-level data. The key concept is the \emph{fundamental causal partition}, which is the coarsest macro-level description of the system that retains all the causal (but not merely correlational) information about it. We show how it can be learned from experimental data.
\item We prove a generalization of the Causal Coarsening Theorem of \citet{Chalupka2015} that now applies to high-dimensional input \emph{and} high-dimensional output. The theorem enables efficient experiment design for learning the fundamental causal partition when experimental data is hard to obtain (but observational data is readily available).
\item We identify the conditions under which it is possible to have causal descriptions of a system at multiple levels of aggregation, and show how these can be learned.
\end{compactenum}

Code that implements our algorithms and reproduces the full simulated experiment is available online at \url{http://vision.caltech.edu/~kchalupk/code.html}.

\subsection{A Motivating Example: Visual Neurons}
\label{sec:example}
Our research is partially inspired by a problem at the core of much of modern neuroscience: Can we detect which features of a visual stimulus result in particular responses of neural populations without pre-defining the stimulus features or the types of population response? 

For example, ~\citet{Rutishauser2011} analyze data from multiple electrodes implanted in human amygdala. The patient is asked to look at images containing either whole human faces, faces randomly occluded with Gaussian ``bubbles'', or images of specific regions of interest in the face---say the eye or the mouth. The neurons are then sorted according to whether they are full-face selective or not, and the response properties of the neurons are analyzed in the two populations. This set-up is an instance of a widely used experimental protocol in the field: prepare stimuli that represent various hypotheses about what the neurons respond to; record from single or multiple units; and analyze the responses with respect to the candidate hypotheses. 

But what if the candidate hypotheses are wrong? Or if they do not line up cleanly with the actually relevant features? Our method proposes a less biased and more automatized process of experimentation: Record neural population responses to a broad set of stimuli. Then jointly analyze what features of the stimuli modify responses of the neural population \emph{and} what features of neural activity are changing in response to the stimuli. To our knowledge, such joint cause-and-effect learning is a novel contribution not only in the neuroscientific setting, but to a whole array of other scientific disciplines.

We will use a simple neural population response simulation as a running example throughout the article. In the simulation (see Fig.~\ref{fig:example}), we observe a population of 100 neurons which act spontaneously using dynamics defined by Izhikevich's equations~\citep{Izhikevich2003}. The equations are designed to reproduce the behavior of human cortical neurons. As the ground-truth structures of interest, we define simple macro-level causes and effects: Presented with an image containing a horizontal bar (h-bar), the ``top half'' of the neural population produces a pulse of joint activity. When presented with a vertical bar (v-bar), the same population synchronizes in a 30Hz rhythm. The remaining (``bottom half'') population acts independently of the visual stimuli (perhaps the experimenter unwittingly placed some of the electrodes in a non-visual brain area). Half the time these ``distractor neurons'' follow their spontaneous noisy dynamics, and half the time they synchronize to produce a rhythmic activity. One can think of this activity as being caused by internal network dynamics, extra-visual stimuli, the animal's hunger or any other cause, as long as it is independent of the image presented by the experimenter. 

\begin{figure}
\centering
\includegraphics{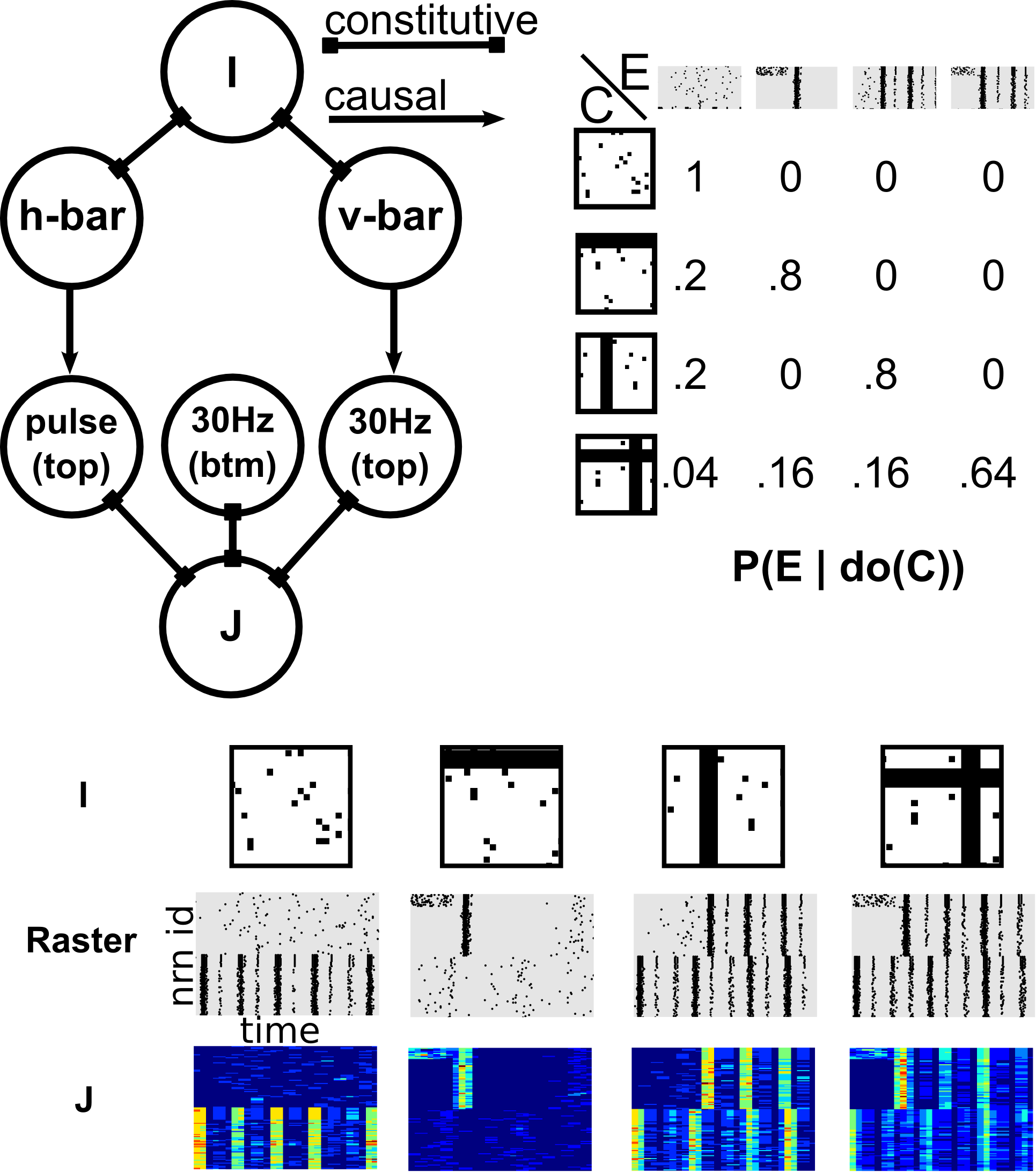}
\caption{\textbf{A simulated neuroscience experiment}. A stimulus image $I$ can contain a horizontal bar (h-bar), a vertical bar (v-bar), neither, or both (plus uniform pixel noise). In response to an image, a simulated population of neurons (the ``top'' population) can produce a single pulse of joint activity, a 30 Hz rhythm, both, or neither, with probabilities $P(\text{pulse} \mid \Do{\text{h-bar}})=0.8$ and $P(\text{30Hz}\mid \Do{\text{v-bar}})=0.8$. These two causal mechanisms compose to yield the full response probability table shown in the top right. In addition, another (``bottom'') population of neurons can exhibit a rhythmic activity independent of the stimulus image. The system's output $J$ is a 10ms-window running average of the neural rasters, with the neuron indices shuffled (as a neuroscientist has no a-priori knowledge of how to order neurons). Here we show example $J$'s sorted by neuron id; we use the shuffled version in our experiments.}
\label{fig:example}
\end{figure}

The example is made up of deliberately simple features for ease of illustration and interpretation. Nevertheless, it hints at what makes similar problems non-trivial to solve. The causal features can be convoluted with salient, probabilistic structure (such as the rhythmic behaviors generated in the ``bottom'' neuronal population). Moreover, the data and its features can be difficult to interpret directly ``by looking'': after reshuffling the neural indices, the raster plots are hardly distinguishable by the human eye, and in many domains (e.g.\ in finance) the data have no special spatial structure, since they can consist simply of rows of numbers.

\section{MACRO-CAUSES AND -EFFECTS}

\citet{Chalupka2015} provide a method to discover from image pixels the macro-level visual cause of a pre-defined macro-level ``target behavior''. In contrast, we do not assume that the macro-level effect (their ``target behavior'') is already specified. Instead, in a generalization of their framework, we simultaneously recover the macro-level cause $C$ and effect $E$ from micro-variable data. Adopting much of their notation, we repeat and generalize their main definitions here, and refer the reader to the original paper for a more detailed explanation.

\subsection{Multi-Level Systems: a Generative Model}

Let $\calI\subset \mathbb{R}^m$ and $\calJ\subset\mathbb{R}^n$ be two finite sets of possibly huge cardinality -- for example, $\calI$ could be the set of all the 100$\times$100 32-bit RGB images\footnote{In this article we will adopt the common practice of referring to such digitalized continuous data as ``high-dimensional''.}. Let $I$ and $J$ be the random variables ranging over those respective sets. We are interested in systems that are well described by the generative model shown in Fig.~\ref{fig:genmodel}, which we call a (causal) \emph{multi-level system}, or ml-system, for reasons that will become evident. In an ml-system, the probability distribution over $I$ is determined by an independent ``noise'' variable $\epsilon_I$ and a (confounding) variable $H$. Both $\epsilon_I$ and $H$ are assumed to be discrete but can have very high cardinality. $J$ is generated analogously, except that it is also caused by $I$. The joint probability distribution over $I$ and $J$ is thus given by:

\begin{equation}
P(J, I) = \sum_H P(J\mid I, H) P(I \mid H) P(H)\notag.\label{eq:genmodel}
\end{equation} 
The independent noise variables $\epsilon_I$ and $\epsilon_J$ are marginalized out and omitted in the above equation for clarity.
\begin{figure}
\centering
\includegraphics{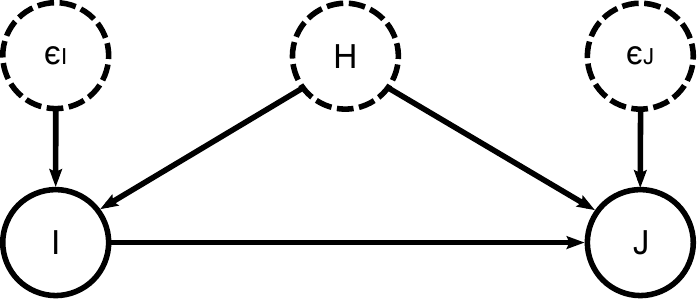}
\caption{\textbf{The generative model of a causal ml-system}. Dashed nodes indicate variables that are not measured. All variables are discrete but can be of huge cardinality. The ``input'' $I$ causes the ``output'' $J$. In addition, the two can be confounded by a hidden variable $H$. Finally, $I$ can contain salient probabilistic structure independent of $J$, and vice-versa.}
\label{fig:genmodel}
\end{figure}

\subsection{The Fundamental Causal Partition}\label{sec:fcp}
An important challenge of causal analysis is to distinguish between dependencies invariant under intervention and those that arise due to confounding. That is, following \citet{Pearl2009} we want to distinguish between the observational conditional probability of $P(Y\mid X)$ for two variables $X$ and $Y$ and the causal probability arising from an intervention on $X$, namely, $P(Y\mid \Do{X})$. An ml-system is sufficiently general to represent dependencies between $I$ and $J$ that remain invariant under intervention and those that are only due to confounding (by $H$). 

In addition, we want to distinguish between micro-variables and the macro-variables that stand in a \emph{constitutive} relation to the micro-variables: An intervention on the micro-variables fixes the macro-variables (for example, the exact spiking time of every neuron determines whether or not a pulse is present), while an intervention on the macro-variable (e.g.\ the presence of a 30Hz neural rhythm) may not uniquely fix the states of the micro-variables that constitute the macro-variable. We follow \citet{Chalupka2015} in first defining a micro-level manipulation, and reserving Pearl's $\Do{}$-operation for the interventions on a macro-variable:

\begin{definition}[Micro-level Manipulation]
A \emph{micro-level manipulation} is the operation $\man{I=i}$ (we will often simply write $\man{i}$ for a specific manipulation) that changes the micro-variables of $I$ to $i\in\calI$, while not (directly) affecting any other variables (such as $H$ or $J$). That is, the manipulated probability distribution of the generative model in Eq.~\eqref{eq:genmodel} is given by 
\[P(J \mid \man{I=i}) = \sum_{H} P(J\mid I=i, H) P(H).\]
\end{definition}

Our goal is to define (and then learn) the most compressed description of an ml-system that retains all the information about the causal effect of $I$ on $J$, that is, we want the most efficient description of the possible interventions and their effects in the system.

\begin{definition}[Fundamental Causal Partition, Causal Class]
Let $(\calI, \calJ)$ be a causal ml-system. The \emph{fundamental causal partition of $\calI$}, denoted by $\Pi_c(\calI)$ is the partition induced by the equivalence relation $\sim_I$ such that 
\begin{align*}
i_1 \sim_I i_2 & \quad \Leftrightarrow \quad \forall_{j\in \calJ} P(j\mid \man{i_1}) = P(j\mid \man{i_2}).
\end{align*}
Similarly, the \emph{fundamental causal partition of $\calJ$}, denoted by $\Pi_c(\calJ)$, is the partition induced by the equivalence relation $\sim_J$ such that 
\begin{align*}
j_1 \sim_J j_2 \quad \Leftrightarrow \quad \forall_{i\in\calI}\;P(j_1\mid \man{i}) = P(j_2 \mid \man{i}). 
\end{align*}
We call a cell of a causal partition a \emph{causal class} of $I$ or $J$.
\end{definition}

In words, two elements of $\calI$ belong to the same causal class if they have the same causal effect on $J$. Two elements of $\calJ$ belong to the same causal class if they arise equally likely after any micro-level manipulation of $I$. The causal classes are thus good candidates for our causal macro-variables:

\begin{definition}[Fundamental Cause and Effect]
In a causal ml-system $(\calI, \calJ)$, the \emph{fundamental cause} $C$ is a random variable whose value stands in a bijective relation to the causal class of $I$. The \emph{fundamental effect} $E$ is a random variable whose value stands in a bijective relation to the causal class of $J$. We will also use $C$ and $E$ to denote the functions that map each $i$ and $j$, respectively, to its causal class.\footnote{In a slight abuse of terminology we will at times use the causal macro-variables to refer to their (bijectively) corresponding partitions, for example, ``$\bar{C}$ is a coarsening of $C$''.} 
\end{definition}

When the fundamental cause and effect are non-trivial, i.e.\ when their values correspond to non-singleton sets of micro-states, then we refer to them as \emph{causal macro-variables}. Figure~\ref{fig:example} illustrates the ground-truth fundamental cause and effect in our simulated neuroscience experiment. The cause $C$ has four states: presence of a vertical bar (v-bar), presence of a horizontal bar (h-bar), presence of both and presence of neither in the image $I$. $C$ causes the effect $E$, which also has four states: presence of pulse, rhythm, both or neither in the activity of a population of neurons in a raster plot. The precise details of these structures (locations of the bars; exact neural spiking times) are irrelevant to the causal interactions in the system, as are the uniform noise in the stimulus images or the strong rhythm generated by the ``bottom'' population of neurons. Despite being an aggregate of micro-variables, $C$ is a well-defined ``causal variable'' as used in the standard framework of causal graphical models. We can define, in a principled way, an intervention on it (analogously for $E$):

\begin{definition}[Macro-level Causal Intervention]
The operation $\Do{C=c}$ on a macro-level cause is given by a manipulation of the underlying micro-variable $\man{I=i}$ to some value $i$ such that $C(i)=c$.
\end{definition}

We can now state the first part of a two-part theorem that justifies the name \emph{fundamental} causal partition. Intuitively, knowing the fundamental causal partition of a system tells us everything there is to know about the causal mechanism implicit in $P(J \mid \man{I})$: Any coarser partition loses some information, any finer partition contains no more causal information.

\begin{theorem}[Sufficient Causal Description, Part 1]
Let $(\calI, \calJ)$ be a causal ml-system and let $E$ be its fundamental causal effect. Let $\mathbf{E}$ be $E$ applied sample-wise to a sample from the system (so that e.g. $\mathbf{E}(j_1,\cdots,j_k)=(E(j_1), \cdots, E(j_k))$). Then among all the partitions of $\calJ$,  $\mathbf{E}$ is the minimal sufficient statistic for $P(J\mid\man{i})$ for any $i\in\calI$. \label{thm:minimal_suff}
\end{theorem}

The proof (in Supplementary Material) is a standard application of Fisher's factorization theorem. Unfortunately, the theorem does not do justice to the intuition that the fundamental cause, too, compresses information about the causal mechanisms of the system. However, unless we assume a distribution $P(\man{I})$ over the interventions, we cannot apply the notion of a sufficient statistic to manipulations in the $\calI$ space. Following Pearl's approach, we refrain from specifying intervention distributions and instead return to this question using a different technique in Sec.~\ref{sec:composite}.

\section{LEARNING THE FUNDAMENTAL CAUSAL STRUCTURE}
We first show how to learn causal macro-variables from \emph{experimental} data, sampled directly from $P(J\mid\man{I})$. Experimental data is generally costly to obtain, so in the following section we  prove the Fundamental Causal Coarsening Theorem that shows one can use \emph{observational} data sampled according to $P(J\mid I)$ to minimize the number of experiments needed to establish the fundamental causal partitions. 

\subsection{Learning With Experimental Data}
\label{sec:experiments}
Consider a dataset $\{(i,j)\}$ of size $N$ generated experimentally from a causal ml-system $(\calI, \calJ)$: each $i$ is chosen by the experimenter arbitrarily, and each $j$ is generated from $P(J\mid \man{i})$. Algorithm~\ref{alg:fundamental_learning} takes such data as input, and computes the fundamental cause and effect of the system. We relegate the detailed discussion of the algorithm (as well as the details of our implementation) to Supplementary Material. Here, instead, we provide a step-by-step illustration of the algorithm's application to the simulated neuroscience problem from Sec.~\ref{sec:example}.

\IncMargin{0.5em}\begin{algorithm}[t!]
\caption{\textbf{Learning the Fundamental Cause and Effect}}
\label{alg:fundamental_learning}
\SetKwFunction{DimReduce}{DmR}
\SetKwFunction{DensLearn}{DensLearn}
\SetKwFunction{Cluster}{Clstr}
\SetKwFunction{Classify}{Clsfy}

\SetKwData{Eft}{Eft}
\SetKwData{Css}{Css}

\SetKwInOut{Input}{input}
\SetKwInOut{Output}{output}

\Input{$\mathcal{D}_{csl} = {\{(i_1, j_1)},\cdots,(i_N, j_N)\}$ -- causal data.\\
       $j_k \sim P(J \mid \man{i_k})$, $1\leq j\leq N$. \\
       $\DensLearn$ -- a density learning routine.\\
       $\Cluster$ -- a clustering routine.\\
       $\Classify$ -- a classification routine.}
\Output{$C\colon \calI\to \{1,\cdots,S_C\}$ -- the fundamental cause.\\
        $E\colon \calJ\to \{1,\cdots,S_E\}$ -- the fundamental effect.\\}
\BlankLine
$P_{J\mid\hat{I}} \leftarrow \DensLearn(\mathcal{D}_{csl})$\;
$\textit{Eft}_{mic} \leftarrow \{[P_{J\mid\hat{I}}(i,j_1),\cdots,P_{J\mid\hat{I}}(i,j_N)] \mid i\in\calI\}$\;\label{alg:fundamental_learning:effect}
$\textit{Cs}_{mic} \leftarrow \{[P_{J\mid\hat{I}}(i_1,j),\cdots,P_{J\mid\hat{I}}(i_N,j)] \mid j\in\calJ\}$\;\label{alg:fundamental_learning:cause}
$C' \leftarrow \Cluster(\textit{Eft}_{mic})$\tcp*{$\text{range}(C')=\{1,\cdots,S_C\}$}
$E' \leftarrow \Cluster(\textit{Cs}_{mic})$\tcp*{range($E'$)=$\{1,\cdots,S_E\}$}
$\textit{Eft}_{mac} \leftarrow \{[P(e_1|c),...,P(e_{S_E}|c)] \mid c=1,...,S_C\}$\;
$\textit{Cs}_{mac} \leftarrow \{[P(e| c_1),..., P(e| c_{S_C})] \mid e=1,...,S_E\}$\;
Merge $C'$ clusters with similar $\textit{Eft}_{mac}$ values.\;\label{alg:fundamental_learning:cmerge1}
Merge $E'$ clusters with similar $\textit{Cs}_{mac}$ values.\;\label{alg:fundamental_learning:cmerge2}
$C \leftarrow \Classify((i_1, C'(i_1)), \cdots, (i_N, C'(i_N)))$\;
$E \leftarrow \Classify((j_1, E'(j_1)), \cdots, (j_N, E'(j_N)))$\;
\end{algorithm}\DecMargin{0.5em}

We generated 10000 images $i$ similar to those shown in Fig.~\ref{fig:example}: 2500 h-bar images (with varying h-bar locations and uniform pixel noise), 2500 v-bar images, 2500 ``h-bar + v-bar'' images and 2500 uniform noise images. Of course, this is an ideal dataset that we can only design because we know the ground-truth causal features. In practice, the  experimenter would want to choose as broad a class of stimuli as reasonable. Next, for each image we generated a corresponding time-averaged, neuron-index-shuffled raster plot $j$ according to $P(J\mid \man{i})$. We then applied Alg.~\ref{alg:fundamental_learning} to this experimental data. The output is for each image $i$ an estimate of its causal class $C(i)$, and for each raster $j$ an estimate of its effect class $E(j)$, as defined in Fig.~\ref{fig:example}.
\begin{figure}
\centering
\includegraphics{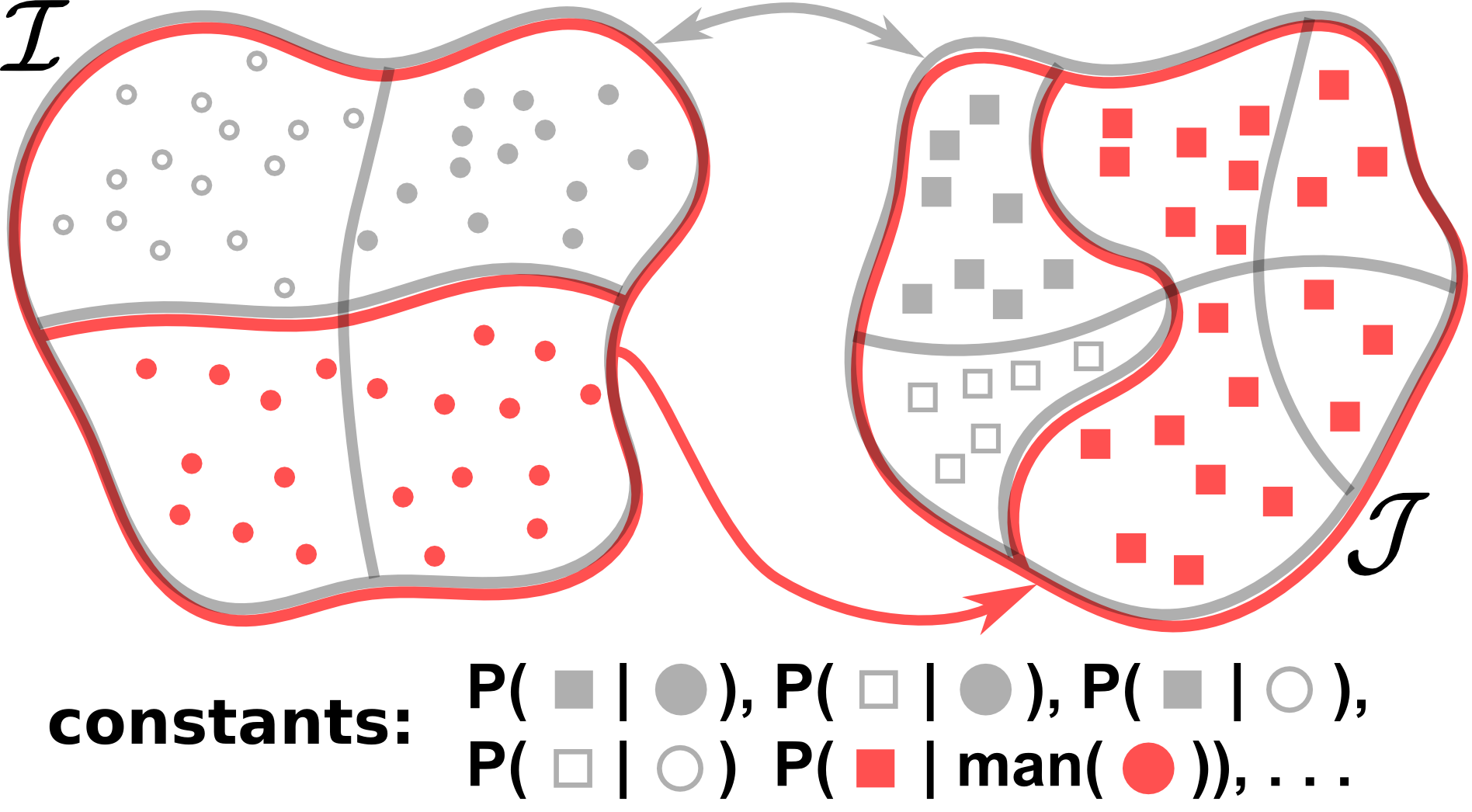}
\caption{\textbf{The Fundamental Causal Coarsening Theorem (fCCT)}. Gray lines delineate the observational partitions on $\calI$ and $\calJ$. Observational probabilities are constant within the gray regions: for any pair $i_1, i_2$ belonging to the same $\calI$ region, and any pair $j_1, j_2$ belonging to the same $\calJ$ region, $P(j_1 \mid i_1)=P(j_1 \mid i_2)=P(j_2 \mid i_1)=P(j_2 \mid i_2)$. Red lines delineate the causal partition: within each red region, probabilities of causation $P(j \mid \man{i})$ are equal. fCCT states that in general, the causal partition coarsens the observational partition, as in the picture.}
\label{fig:cct}
\end{figure}

 Figure~\ref{fig:experiment} shows how Alg.~\ref{alg:fundamental_learning} recovers the macro-variable causal mechanism of our simulated single-unit-recording experiment. Three remarks are in order:

\begin{compactenum}
\item For purposes of illustration, the macro-level causal variables are very simple. Nevertheless, the procedure is completely general and could be applied to detect causal macro-variables that do not admit such a simple description. We believe the method holds promise for applications in a broad set of scientific domains.
\item The algorithm does not simply cluster $\calI$ and $\calJ$. Instead, it clusters the probabilistic effects of points in $\calI$, and the probabilities of causation for points in $\calJ$. Its crucial function is to ignore any structures that are not related to the causal effect of $I$ on $J$. In our example, the raster plots contain salient structure that is causally irrelevant: With probability 0.5, the ``bottom'' subpopulation of neurons spikes in a synchronized rhythm. Simply clustering $\calJ$ would sub-divide the true causal classes in half. Fig.~\ref{fig:experiment}e shows that the algorithm finds the correct solution.
\item There are many possible alternatives to Alg.~\ref{alg:fundamental_learning}, each with different advantages and disadvantages. The particular solution we chose is a direct application of the definitions, and works well in practice. However, it does introduce additional assumptions --- in particular, $p(J\mid \man{I})$ needs to be smooth both as a function of $j$ and $i$ for the algorithm to work perfectly.
\end{compactenum}

\subsection{The Fundamental Coarsening Theorem and Experiment Design}\label{sec:fct}
If only data sampled from $P(J \mid I)$ is available, it is in general impossible to determine the fundamental causal partition. The causal effect from $I$ to $J$ cannot always be separated from the confounding due to $H$ (recall Fig.~\ref{fig:example}). Instead, we can directly apply Alg.~\ref{alg:fundamental_learning} to the observational data to obtain the \emph{observational partition} of a causal ml-system:

\begin{definition}[Observational Partition, Observational Class]
Let $(\calI, \calJ)$ be a causal ml-system. The \emph{observational partition of $\calI$}, denoted by $\Pi_o(\calI)$, is the partition induced by the equivalence relation $\sim_I$ such that $i_1 \sim_I i_2$ if and only if $P(J\mid I=i_1) = P(J\mid I=i_2).$ The \emph{observational partition of $\calJ$}, denoted by $\Pi_o(\calJ)$, is the partition induced by the equivalence relation $\sim_J$ such that $j_1 \sim_J j_2$ if and only if $\forall_{i\in\calI}\;P(j_1\mid i) = P(j_2 \mid i)$.
A cell of an observational partition is called an \emph{observational class} (of $\calI$ or $\calJ$).
\end{definition}

Spurious correlates can introduce structure in $\Pi_o$ that is irrelevant to $\Pi_c$ (see Eq.~\eqref{eq:genmodel} and the discussion on spurious correlates in \citet{Chalupka2015}). Nevertheless, we can aim to minimize the number of experiments needed to obtain the fundamental causal partition. The following theorem (which generalizes the Causal Coarsening Theorem of~\citep{Chalupka2015}) shows that in general, observational data can be efficiently transformed into causal knowledge about an ml-system.

\begin{theorem}[Fundamental Causal Coarsening] Among all the generative distributions of the form shown in Fig.~\ref{fig:genmodel} which induce given observational partitions $(\Pi_o(\calI), \Pi_o(\calJ))$:
\begin{compactenum}
\item The subset of distributions that induce a fundamental causal partition $\Pi_c(\calI)$ that is \emph{not} a coarsening of the observational partition $\Pi_o(\calI)$ is Lebesgue measure zero, and 
\item The subset of distributions that induce a fundamental causal partition $\Pi_c(\calJ)$ that is \emph{not} a coarsening of the $\Pi_o(\calJ)$ is Lebesgue measure zero.
\end{compactenum}
\end{theorem}

In other words, the observational partition over $\calI$ may subdivide come cells of the causal partition, but not \textit{vice-versa}, and the observational partition over $\calJ$ may subdivide some cells of the causal partition, but not \textit{vice-versa}. Fig.~\ref{fig:cct} illustrates the Fundamental Causal Coarsening Theorem (fCCT).

\begin{figure*}[t!]
\centering
\includegraphics{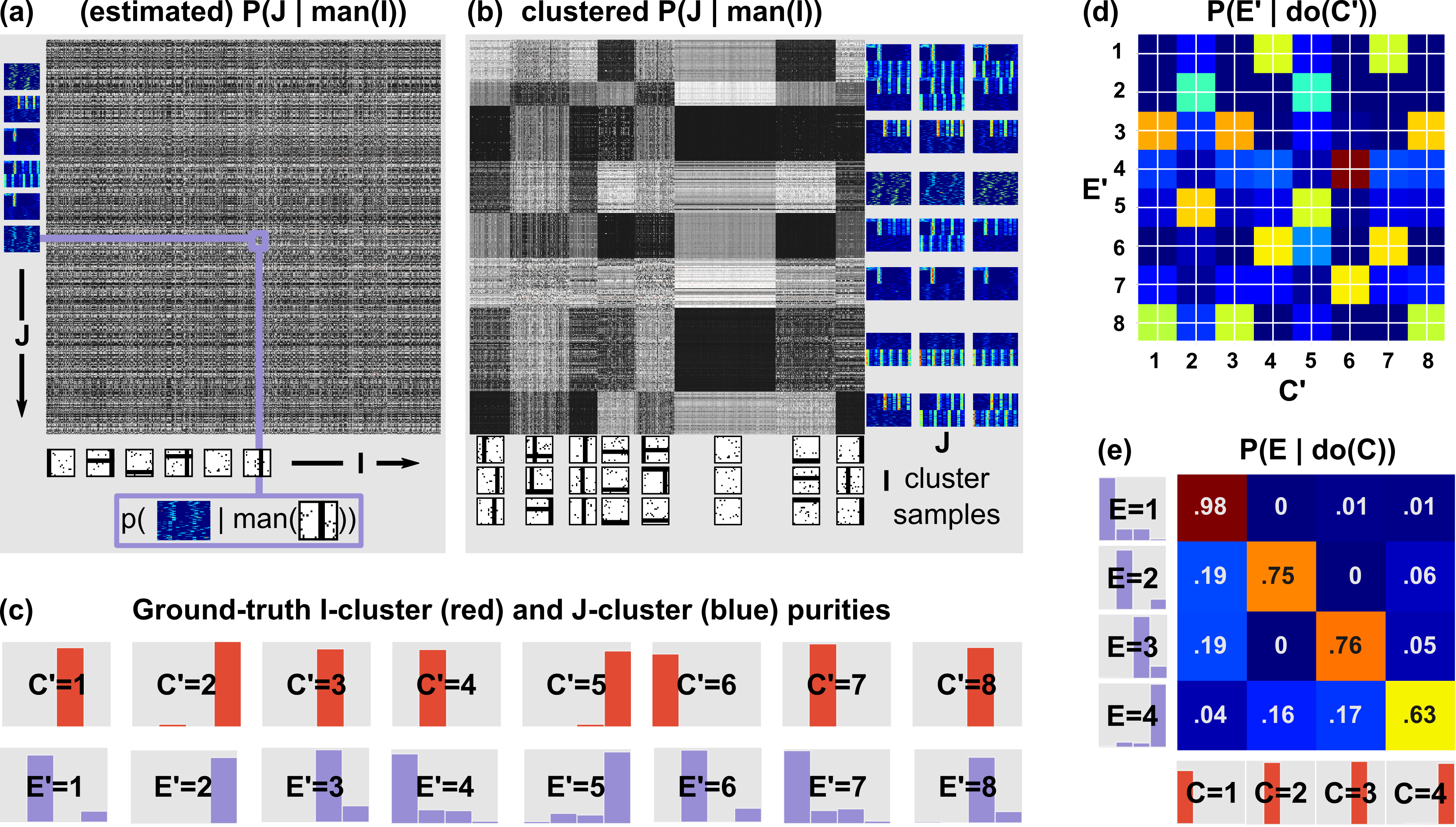}
\caption{\textbf{Learning the fundamental causal partition}. The figure demonstrates Algorithm~\ref{alg:fundamental_learning} applied to the example from Fig.~\ref{fig:example}. \textbf{(a)} Given a dataset $\{(i_k, j_k)\}_{k=1...N}$, the algorithm learns data density $P(j\mid \man{i})$ and forms a matrix in which the $kl$-th entry is the estimated $P(j_l\mid\man{i_k})$. \textbf{(b)} The rows and columns of the matrix are clustered. Each cluster of rows corresponds to a cell of $C'$, the proposed fundamental partition of $\calI$, and each cluster of columns corresponds to a cell of $E'$, the proposed fundamental partition of $\calJ$.  \textbf{(c)} The histograms show the ground-truth causal class of the points in each cluster (this ground truth is unknown to the algorithm). For example, the cell $E' = 8$ contains a majority of raster plots that contain the ``30Hz (top)'' causal structure; it also contains some ``30Hz (top) + pulse'' rasters. \textbf{(d)} The algorithm computes the probability table $P(E' \mid \Do{C'})$ by counting the co-occurrences of the cluster labels. \textbf{(e)} Finally, the rows of this table are merged according to their similarity to form the fundamental partition $\Pi_C$ of the data, and the columns are merged to form $\Pi_E$. For example, rows $E'=1$ and $E'=3$ of the table in (d) are similar---indeed, the cluster purity histograms indicate that both rows correspond to sets of images with a vertical bar. $P(E\mid\Do{C})$ is very similar to the ground-truth table (see Fig.~\ref{fig:example}), and the final $C, E$ clusters are pure (as shown along the axes of the  table).}
\label{fig:experiment}
\end{figure*}

We prove the theorem in Supplementary Material. fCCT suggests an efficient way to learn causal features of a system starting with observational data only: first, learn the observational partition using Alg.~\ref{alg:fundamental_learning}. Next, pick (at least) one $i$ belonging to each observational class and estimate $P(J \mid \man{i})$. To obtain the causal partition, merge the observational cells whose $i$'s induce the same distribution over $J$. Then, pick at least one $j$ from each observational class and merge the cells whose $j$'s induce the same $P(j\mid \man{i})$ for all $i\in\calI$.

\section{SUBSIDIARY CAUSES AND EFFECTS}
\label{sec:composite}
The behavior of our simulated neural population is affected by two independent causal mechanisms: the presence of a v-bar can create a neural pulse, and the presence of an h-bar can induce a 30Hz neural rhythm. We wrote ``$P(\text{30Hz}\!=\!1 \mid \Do{\text{v-bar}\!=\!1}) = .8$ and $P(\text{pulse}\!=\!1 \mid \Do{\text{h-bar}\!=\!1}) = .8$'', and said that these two mechanisms compose to bring about the observed effects. We now formalize under what conditions higher-level variables, such as ``30Hz'' or ``v-bar'', can arise from the fundamental causal partition.

\begin{definition}[Subsidiary Causal Variables]
\label{def:subsidiary}
Let $C$ and $E$ be the fundamental cause and effect of a causal ml-system. Let $\bar{C}$ and $\bar{E}$ be strict coarsenings of $C$ and $E$. Denote by $c_1(l), \cdots, c_{N_l}(l)$ the cells of $C$ that belong to the $l$-th cell of $\bar{C}$. We say that $\bar{C}$ and $\bar{E}$ are  \emph{subsidiary causal variables}, and that $\bar{C}$ is a \emph{subsidiary cause} of the \emph{subsidiary effect} $\bar{E}$ if (i) $\forall_{l} P(\bar{E}\mid\Do{C=c_1(l)}) = \cdots = P(\bar{E}\mid\Do{C=c_{N_l}(l)})$, and (ii) $P(\bar{E}\mid\Do{\bar{C}=\bar{c}_1}) \neq P(\bar{E}\mid\Do{\bar{C}=\bar{c}_2})$ for any distinct $\bar{c}_1$ and $\bar{c}_2$ in the  range of $\bar{C}$.
\end{definition}

According to the definition, any coarsening of $C$ and $E$ that aspires to be a subsidiary cause-effect pair has to satisfy two conditions. 
First, manipulations on the subsidiary cause $\bar{C}$ have to be well-defined. The definition guarantees that any two $i_1, i_2$ for which $\bar{C}(i_1) = \bar{C}(i_2)$ generate the same distribution over the subsidiary effect, so that $P(\bar{E}\mid \Do{\bar{C}=\bar{C}(i_1)})=P(\bar{E}\mid \Do{\bar{C}=\bar{C}(i_2)})$. In our example, producing an image with an h-bar induces the neural pulse with probability $.8$. The probability of the pulse is indifferent to the presence/absence of a v-bar (or any other structure) in the image (see also Fig.~\ref{fig:subsidiary}a,b). On the other hand, we claimed that v-bars cause rhythms, not pulses (see Fig.~\ref{fig:subsidiary}c). What shows formally that v-bars do not cause pulses? Producing an image $i$ with a v-bar but no h-bar gives us $P(\text{pulse}\mid \man{i})=0$, but if $i$ contains both h- and v-bars, we have $P(\text{pulse}\mid\man{i})=.8$. This disagrees with our definition of what it takes to be a causal variable: the manipulation on the macro-cause v-bar is not well-defined with respect to the macro-effect pulse, as the effects of micro-variables belonging to the same macro-variable causal class are not the same. We have what \citet{Spirtes2004} call an ``ambiguous manipulation'' of v-bar with respect to the pulse. 

\IncMargin{0.5em}\begin{algorithm}[t!]
\caption{\textbf{Finding Subsidiary Variables}}
\label{alg:subsidiary}
\SetKwFunction{Part}{Partitions}
\SetKwFunction{Range}{range}
\SetKwFunction{Effect}{effect}

\SetKwInOut{Input}{input}
\SetKwInOut{Output}{output}

\Input{$C,E$ -- the fundamental cause and effect (and the corresponding partitions).}
\Output{$\mathcal{S}=(C^1, E^1),\cdots,(C^N, E^N)$ -- subsidiary variables of the system.}
\BlankLine
$\mathcal{S}\leftarrow \emptyset$\;
$c_1,\cdots,c_m \leftarrow \Range(C)$\;
$e_1,\cdots,e_n \leftarrow \Range(E)$\;

\For{$\bar{E}\in\Part(E)$}{
  \For{$\bar{e}\in\Range(\bar{E})$}{
    $P(\bar{e}\mid\Do{C=c_k})\leftarrow \displaystyle{\sum_{e_l\in\bar{e}}}P(e_l \mid \Do{C=c_k})$\;\label{alg:subsidiary:pbaredoc}
  }
  Define $\Effect\colon c_k \mapsto P(\bar{E}\mid\Do{C=c_k})$\;
  Let $c_i\sim_{\bar{C}} c_j \Leftrightarrow \Effect(c_i)=\Effect(c_j)$\;
  $\Pi_{\bar{C}}\leftarrow$ partition of $\Range{C}$ induced by $\sim_{\bar{C}}$\;
  $\bar{C}\leftarrow$ random variable corresponding to $\Pi_{\bar{C}}$\;
  $\mathcal{S} \leftarrow \mathcal{S}\cup (\bar{C}, \bar{E})$\;
}
\end{algorithm}\DecMargin{0.5em}

The  second condition in the subsidiary variable definition ensures that the values of subsidiary causes are only distinct when they have distinct effects. A succinct answer to the question ``what causes the neural pulse?'' is ``the presence of a horizontal bar'' --- not ``two states: one corresponding to the presence of a horizontal bar along with the presence of a vertical bar; the other to the presence of a horizontal bar without the presence of a vertical bar''. The two states have the exact same probabilistic effect, and therefore should be combined to one.

Together, the  two conditions ensure that subsidiary causes and effects allow for well-defined, parsimonious manipulations. Equipped with the notion of subsidiary causal variables and an understanding of what it takes to define $P(\bar{E}\mid\Do{\bar{C}})$, we can complete our Sufficient Causal Description theorem: 

\begin{theorem}[Sufficient Causal Description, Part 2]
The fundamental causal variables $C$ and $E$ losslessly recover $P(j\mid \man{i})$. No other (subsidiary) causal variables losslessly recover $P(j\mid\man{i})$. Any other partition of $(\calI, \calJ)$ is either finer than $C, E$ or does not define unambiguous manipulations. In this sense, the fundamental causal partition corresponds to the coarsest partition that losslessly recovers $P(j\mid\man{i})$.
\label{thm:minimal_compr}
\end{theorem}

The proof is provided in Supplementary Material. The theorem suggests that the use of subsidiary variables is to \emph{ignore} causal information that is not of interest. For example, having discovered the fundamental effects of images on a brain region the neuroscientist might want to focus on the subsidiary effects whose analogues were observed in other brain regions, or in other animals. Alg.~\ref{alg:subsidiary} shows a simple (yet combinatorially expensive) procedure to discover the full set of subsidiary causes and effects in an ml-system. The algorithm iterates over all the possible coarsenings of $E$, the fundamental effect, and computes, for each, the corresponding coarsening (not necessarily strict) of the fundamental cause that adheres to Def.~\ref{def:subsidiary}.

To complete the picture of how the fundamental and subsidiary variables relate to each other, we formalize the intuition that the fundamental causal partition can be a product of its subsidiary variables. Recall that we have defined causal macro-variables as partitions of sets of values of random micro-variables. The composition of causal variables is defined in terms of the product of partitions.

\begin{definition}[Partition Product, Macro-Variable Composition]
Let $\Pi_1$ and $\Pi_2$ be partitions of the same set $X$. The product of the partitions, denoted $\Pi_1 \otimes \Pi_2$, is the coarsest partition of $X$ that is a refinement of both $\Pi_1$ and $\Pi_2$. The set of partitions of $X$ forms a commutative monoid under $\otimes$. The composition $C$ of two causal macro-variables $C_1$ and $C_2$ is defined as the product of the corresponding partitions. In this case, we will use the $\otimes$ operator to write $C=C_1\otimes C_2$.
\end{definition}

Finally, we describe a special class of subsidiary variables to gain additional insight into the fundamental causal structure of ml-systems.

\begin{definition}[Non-Interacting Subsidiary Variables]
Let $C^1, C^2$ be subsidiary causes with respective subsidiary effects $E^1, E^2$. Denote by $(e_1,e_2)$ the cell of $E^1\otimes E^2$ that corresponds to the intersection of a cell $e^1$ of $E^1$ and cell $e_2$ of $E^2$, and analogously for $(c_1, c_2)$. $C^1$ and $C^2$ are \emph{non-interactive} if for any non-empty $(c_1,c_2)$ and $(e_1, e_2)$ we have $P(E^1\otimes E^2\!=\!(e_1,e_2)\mid \Do{C^1\otimes C^2\!=\!(c_1,c_2)}) = P(E^1\!=\!e_1\mid\Do{C^1\!=\!c_1})\times P(E^2\!=\!e_2\mid\Do{C^2\!=\!c_2})$.
\end{definition}

Among all the possible ml-systems, the fundamental causal partition gives rise to \emph{no subsidiary causes} in almost all the cases. The presence of coarse, non-interacting subsidiary causes (such as the h-bar and the v-bar in our example) is a strong indicator of independent physical causal mechanisms that produce symmetries in the fundamental causal structure of the system. Our framework enables the scientist to automatically detect such independent mechanisms from data.

For example, let $C^1$= ``presence of h-bar'', $C^2$= ``presence of v-bar'', $E^1$= ``presence of pulse'', $E^2$= ``presence of rhythm (top)''. We can discover these variables in from data using Alg.~\ref{alg:fundamental_learning} followed by Alg.~\ref{alg:subsidiary}, and check that indeed indeed they are non-interacting. In fact, these two subsidiary variables compose to yield the fundamental causal partition and its probability table--we can write $C=C^1\otimes C^2$ and $E=E^1\otimes E^2$ (see Fig.~\ref{fig:subsidiary}d). 

\begin{figure}
\centering
\includegraphics{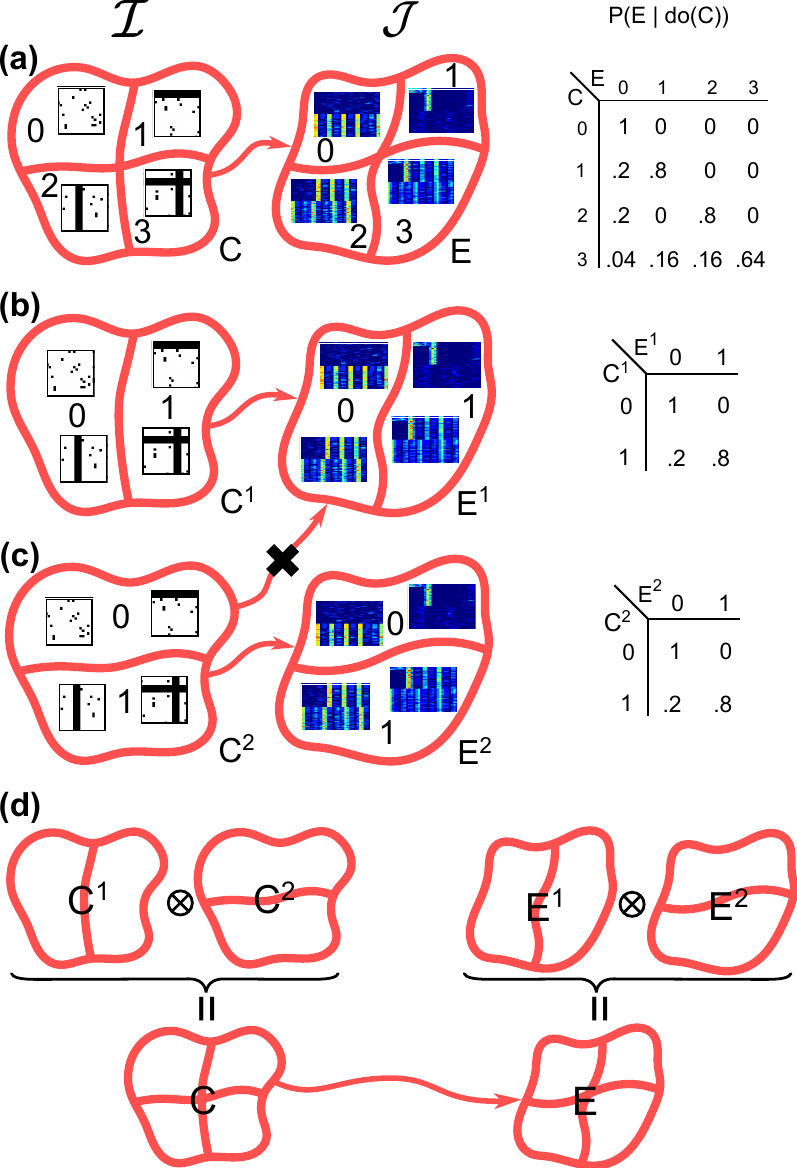}
\caption{\textbf{Subsidiary Causal Variables}. \textbf{(a)} The fundamental cause and effect of the neuroscience example from Fig.~\ref{fig:example}. \textbf{(b)} The subsidiary cause $C^1$, or ``presence of an h-bar''. The corresponding coarsening of $C$ groups together all images which contain no h-bars ($C^1=0$) and all the images which contain an h-bar ($C^1=1$). Similarly, the subsidiary effect of $C^1$ groups together raster plots with and without the ``pulse'' behavior. 
\textbf{(c)} The subsidiary cause $C^2$, or ``presence of a v-bar'' and its effect $E^2$. Note that $E^1$, for example, is \emph{not} an effect of $C^2$. If it was, the effects of manipulations $\Do{C^2=0}$ as well as $\Do{C^2=1}$ would be ambiguous: $P(E^1=1 \mid \Do{C^2=1})$ could be either $.8$ or $0$, depending on whether the manipulated micro-variable contains an h-bar or not. \textbf{(d)} $C^1$ and $C^2$ are non-interacting subsidiary causes. The effect of their product is the product of their effects. In fact, their composition forms the fundamental causal partition of the system.}
\label{fig:subsidiary}
\end{figure}

\section{DISCUSSION AND CONCLUSION}\label{sec:discussion}
In general it is possible that macro-variable causes and effects are barely coarser (if at all) than the corresponding micro-variables. The hope that $C$ and $E$ have a ``manageable'' cardinality, such as those in Fig.~\ref{fig:example}, is similar in spirit to standard assumptions in both supervised and unsupervised learning. There, a set of continuous data is clustered into a discrete number of subsets according to some feature of interest. Here the ``feature of interest'' is the causal relationship between $C$ and $E$. 

Given that the discussion of macro-causal relations is commonplace in scientific discourse, we take the scientific endeavors mentioned in the introduction to be predicated on the assumption that micro-level descriptions are not all there is to the phenomena under investigation. Whether or not there in fact are macro-level causes that justify such an assumption is, in light of our theoretical account, an empirical question since -- taking the definitions literally -- macro-causes cannot be defined arbitrarily. When applying the theory to practical cases one needs to assume that micro-variables do not lump together atoms that belong to different ``true fundamental partition'' cells. What happens when this assumption is violated is an open question.

Our approach to the automated construction of causal macro-variables is rooted in the theory of computational mechanics \citep{Shalizi2001,Shalizi2001a,Shalizi2004}. Even though we have focused on learning from experimental data, we cleanly account for the interventional/observational distinction that is central to most analyses of causation. This distinction is entirely lost in heuristic approaches, such as that of \citet{Hoel2013}. Finally, we note that our work is orthogonal to recent efforts to learn causal graphs over high-dimensional variables \citep{Entner2012,Parviainen2015}. Given a directed causal edge between two such variables, our method can extract a macro-variable representation of the relationship, ultimately simplifying the causal graph.

Altogether, we have an account of how causal variables can be identified that does not rely on a definition obtained from domain experts. Given its theoretical generality, we expect our method to be useful in many domains where micro-level data is readily available, but where the relevant causal macro-level factors are still poorly understood. \vspace{-.1in}
\subsubsection*{Acknowledgements}
KC’s and PP’s work was supported by the
ONR MURI grant N00014-10-1-0933. FE would like to
thank Cosma Shalizi for pointers to many relevant results
this paper builds on.

\pagebreak
\section{SUPPLEMENTARY MATERIAL: SUFFICIENT CAUSAL DESCRIPTION THEOREM, PARTS 1 AND 2}
\begin{theorem}[Sufficient Causal Description]
Let $(\calI, \calJ)$ be a causal ml-system and let $C$ and $E$ be its fundamental cause and effect. Let $\mathbf{E}$ be $E$ applied sample-wise to a sample from the system (so that e.g. $\mathbf{E}(j_1,\cdots,j_k)=(E(j_1), \cdots, E(j_k))$). Then:
\begin{compactenum}
\item Among all the partitions of $\calJ$,  $\mathbf{E}$ is the \emph{minimal} sufficient statistic for $P(J\mid\man{i})$ for any $i \in \calI$, and
\item $C$ and $E$ losslessly recover $P(j\mid \man{i})$. No other (subsidiary) causal variable losslessly recovers $P(j\mid\man{i})$. Any other partition is either finer than $C, E$ or does not define unambiguous manipulations. In this sense, the fundamental causal partition corresponds to the coarsest partition that losslessly recovers $P(j\mid\man{i})$.
\end{compactenum}

\end{theorem}

\begin{proof}
1. We first prove that $\mathbf{E}$ is a sufficient statistic. Recall that we assumed $\calJ$ to be discrete, although possibly of vast cardinality. For any $j_k\in\calJ$, write $P(j_k\mid\man{i}) = p_{j_k}$ for the corresponding categorical distribution parameter. Let $\texttt{range}(E)=\{E_1,\cdots,E_M\}$ be the set of causal classes of $J$. By Definition~3 there is a number of ``template'' probabilities $p_{E_1},\cdots, p_{E_M}$ such that $p_{j_k}=p_{E_k}$ if and only if $E(j_k)=E_k$. Consider an i.i.d. sample $\mathbf{j}=j_1,\cdots,j_l$ from $P(J\mid\man{i})$. Then
\begin{align*}
P(j_1,\cdots,j_l\mid\man{i}) &=\Pi_{k=1}^lp_{j_k}\\
                            &=\Pi_{m=1}^Mp_{E_m}^{\#(E_m)},
\end{align*}
where $\#(E_m)\triangleq\Sigma_{k=1}^l\mathbbm{1}\{E(j_k)==E_m\}$ is the number of samples with causal class $E_m$. Since the sample density depends on the samples only through $C$ and $E$ it follows from the Fisher's factorization theorem that $\mathbf{E}$ is a sufficient statistic for $P(J\mid\man{i})$ for any $i\in\calI$. 

Now, we prove the minimality of $E$ among all the partitions of $\calJ$. Consider first any refinement of $E$. One can directly apply the reasoning above to show that the cell assignment in such a partition is also a sufficient statistic. However, any refinement is not the \emph{minimal} sufficient statistic, as the fundamental causal partition is its coarsening--- and thus also its function. Now, consider any partition that is not the fundamental causal partition, and is not its refinement. Call it $E'$. Assume, for contradiction, that $\mathbf{E'}$ is a sufficient statistic for $P(J\mid\man{i})$. Then, by the factorization theorem, $P(j_1,\cdots,j_k\mid\man{i})$ would factorize as $h(j_1,\cdots,j_k)g(E'(j_1),\cdots,E'(j_k))$, where $h$ does not depend on the parameters $p_{j_l}$. Now, take some $j_1^1, j_1^2$ such that $E(j_1^1)\neq E(j_1^2)$ but $E'(j_1^1)=E'(j_1^2)$ (such a pair must exists since $E'$ is not a refinement of $E$ and is not equal to it). Then
\[\frac{P(j_1^1, j_2, \cdots, j_k\mid\man{i})}{P(j_1^2, j_2, \cdots, j_k\mid\man{i})}=\frac{p_{E(j_1^1)}}{p_{E(j_1^2)}},\]
\begin{align*}
&\frac{P(j_1^1, j_2, \cdots, j_k\mid\man{i})}{P(j_1^2, j_2, \cdots, j_k\mid\man{i})}=\\
&\phantom{WWWWW}=\frac{h(j_1^1,\cdots,j_k)g(E'(j_1^1),\cdots,E'(j_k))}{h(j_1^2,\cdots,j_k)g(E'(j_1^2),\cdots,E'(j_k))}\\
&\phantom{WWWWW}=\frac{h(j_1^1,\cdots,j_k)}{h(j_1^2,\cdots,j_k)}
\end{align*}
which, as already stated, does not depend on the parameters of the distribution -- a contradiction.

2. That $P(J\mid\man{i})$ can be recovered from $C$ and $E$ follows directly from the definition of a causal ml-system and its fundamental causal partition. That it cannot be recovered losslessly from any partition that is not a refinement of $C$ and $E$ follows again from the fact that for any such partitions $C'$ and $E'$ there must be is at least one pair $(i_1, j_1), (i_2, j_2)$ for which $p(E'(j_1)\mid\Do{C'(i_1)}) = p(E'(j_2)\mid\Do{C'(i_2)})$ even though $p(j_1\mid\man{i_1})\neq p(j_2\mid\man{i_2})$.
\end{proof}

We note that the first part of Theorem~1 indicates that $\mathbf{E}$ is only a minimal sufficient statistic among all partitions of $\calJ$, i.e.\ among the set of possible causal variables. It is not the minimal sufficient statistic over all possible sufficient statistics for $P(J \mid \man{i})$. In particular, a histogram is a minimal sufficient statistic for the multinomial distribution and is a function of $\mathbf{E}$, but a histogram does not correspond to a partition of $\calJ$.

\section{SUPPLEMENTARY MATERIAL: DETAILS AND IMPLEMENTATION OF ALGORITHM 1}
First, the algorithm uses a density learning routine to estimate $P(J\mid \man{I})$ given the samples. We don't specify the density learning routine, as that is highly problem-dependent. In our experiments, dimensionality reduction with autoencoders~\citep{Hinton2006} followed by kernel density estimation worked well. More sophisticated approaches are readily available, for example RNADE~\citep{Uria2013}. 
Steps~2 and~3 constitute the core of the algorithm: In Step~2, a vector of (estimated) densities $[P(j_1\mid \man{i}),\cdots,P(j_N\mid \man{i})]$ is calculated for each $i_k$ in the dataset ($1\leq k\leq N$). That is, each $i_k$ corresponds to a vector that contains information about the probability of each $j_l$  ($1\leq l\leq N$) occurring given a manipulation $\man{i_k}$ (note that in the original dataset, $i_k$ might have only appeared as paired with one effect $j_k$, sampled from $P(J\mid \man{i})$). Similarly, Step~3, computes for each $j_l$ a vector of estimated densities of $j_l$ occurring given an intervention on each $i_k$. 

Clustering these vectors (Step 4 \& 5) makes it possible to group together all the $i$'s with similar effects, and all the $j$'s with similar causes --- that is, to learn the fundamental causal partition. The number of cells of the fundamental partition is unknown in advance, but it is safe to over-cluster the data. Our implementation uses the Dirichlet Process Gaussian Mixture Model~\citep{Rasmussen1999} for clustering with a flexible number of clusters, but again the algorithm stays clustering-routine-agnostic.

After the initial clustering it should now be easy to merge clusters belonging to the same true causal class, as the probabilistic patterns of mergeable clusters are expected to be similar. The macro-variable cause/effect probability vectors are estimated in Steps~8 and~9. These are analogues to the micro-variable cause/effect density vectors estimated in Step~2. However, instead of estimating the density of the micro-variable data, they count the normalized histograms of conditional probabilities of the $\calJ$ cluster given the $\calI$ clusters. These histograms are aggregates of large numbers of datapoints, and should smooth out errors in the original density estimation. Thus, even if the original clustering algorithm overestimates the number of cells in the fundamental partition, we can hope to be able to merge them based on similarities in the macro-variable histogram vectors. In our experiment, we merge the macro-variable cause/effect probabilities by thresholding the KL-divergence between any two vectors belonging to the same cluster. However, since the number of datapoints to cluster is likely to be very small, the best solution in practice is to cluster them by hand.

By Step~8, the algorithm returns causal labels for the original data samples. These labeled samples can be used to visualize the fundamental causes and effects using the original data samples. To generalize the fundamental cause and effect to the whole $\calI$ and $\calJ$ space, the algorithm trains a classifier using the original data and the learned causal labels.

\section{SUPPLEMENTARY MATERIAL: THE FUNDAMENTAL CAUSAL COARSENING THEOREM}

\begin{theorem}[Fundamental Causal Coarsening] Among all the generative distributions of the form shown in Fig.~2 (main text) which induce given observational partition $(\Pi_o(\calI), \Pi_o(\calJ))$:
\begin{compactenum}
\item The subset of distributions that induce a fundamental causal partition $\Pi_c(\calI)$ that is \emph{not} a coarsening of the observational partition $\Pi_o(\calI)$ is Lebesgue measure zero, and 
\item The subset of distributions that induce a fundamental causal partition $\Pi_c(\calJ)$ that is \emph{not} a coarsening of the $\Pi_o(\calJ)$ is Lebesgue measure zero.
\end{compactenum}
\end{theorem}
\begin{proof} (1) Let $E$ be the fundamental effect of the system. Then $\Pi_c(\calI)$ and $E$ constitute precisely the ``causal partition'' and ``target behavior'' of the system and $\Pi_o(\calI)$ constitutes the ``observational partition'' of the system, as defined by \citet{Chalupka2015}. Thus, the proof of the Causal Coarsening Theorem by~\citet{Chalupka2015} applies directly and proves (1).

(2) While we cannot directly use the proof of~\citet{Chalupka2015}, we follow a very similar proof strategy. The only difference turns out to be in the details of the algebra. We first lay out the proof strategy. Let $j_1, j_2\in\calJ$. We need to show that if $P(j_1 \mid i) = P(j_2\mid i)$ \emph{for every $i\in \calI$}, then also $P(j_1\mid\man{i}) = P(j_2\mid\man{i})$ for every $i$ (for all the distributions compatible with given observational partition, except for a set of measure zero). The proof is split into two parts: (i)~Express the theorem as a polynomial constraint on the space of all $P(i,j,h)$ distributions. (ii)~Show that the polynomial constraint is not trivial. This, by~\citet{Meek1995}, implies that among \emph{all} $P(i,j,h)$ distributions, the fundamental causal partition on $J$ is a coarsening of the fundamental observational partition. (iii)~Prove that (i) and (ii) apply to ``all the distributions which induce a given observational partition'' by showing that this restriction results in a simple reparametrization of the distribution space.\\

(2i) Let $H$ be the hidden variable of the system, with cardinality $K$; let $J$ have cardinality $N$ and $I$ cardinality $M$. We can factorize the joint on $I, J, H$ as $P(J, I, H)=P(J\mid H,I)P(I\mid H)P(H)$. $P(J\mid H,I)$ can be parametrized by $(N-1)\times K\times M$ parameters, $P(I\mid H)$ by $(M-1)\times K$ parameters, and $P(H)$ by $K-1$ parameters, all of which are independent. 

Call the parameters, respectively, 
\begin{align*}
\alpha_{j,h,i} &\triangleq P(J=j \mid H=h, I=i)\\
\beta_{i,h} &\triangleq P(I=i\mid H=h)\\
\gamma_h &\triangleq P(H=h)
\end{align*}
We will denote parameter vectors as
\begin{align*}
\alpha &= (\alpha_{j_1,h_1,i_1}, \cdots, \alpha_{j_{N-1},h_K,i_M}) \in\mathbb{R}^{(N-1)\times K\times M} \\
\beta &= (\beta_{i_1,h_1}, \cdots, \beta_{i_{N-1}, h_K}) \in\mathbb{R}^{(M-1)\times K}\\
\gamma &= (\gamma_{h_1}, \cdots, \gamma_{h_{K-1}}) \in\mathbb{R}^{K-1},
\end{align*}
where the indices are arranged in lexicographical order. This creates a one-to-one correspondence of each possible joint distribution $P(J,H,I)$ with a point $(\alpha,\beta,\gamma)\in P[\alpha,\beta,\gamma] \subset \mathbb{R}^{(N-1)\times K^2(K-1)\times M(M-1)}$, where $P[\alpha, \beta, \gamma]$ is the $(N-1)\times K^2(K-1)\times M(M-1)$-dimensional simplex of multinomial distributions.

To proceed with the proof, we pick any point in the $P(J\mid H,I) \times P(H)$ space: that is, we fix the values of $\alpha$ and $\gamma$. The only free parameters are now the $\beta_{i,h}$; varying these values creates a subset of the space of all the distributions which we will call 
\begin{equation*}
P[\beta; \alpha,\gamma]=\{(\alpha, \beta,\gamma)\; \mid \; \beta \in [0,1]^{(M-1)\times K}\}.
\end{equation*}
$P[\beta; \alpha, \gamma]$ is a subset of $P[\alpha,\beta,\gamma]$ isometric to the $[0,1]^{(M-1)\times K}$-dimensional simplex of multinomials. We will use the term $P[\beta; \alpha, \gamma]$ to refer both the subset of $P[\alpha,\beta,\gamma]$ and the lower-dimensional simplex it is isometric to, remembering that the latter comes equipped with the Lebesgue measure on $\mathbb{R}^{(M-1)\times K}$.

Now we are ready to show that the subset of $P[\beta; \alpha, \gamma]$ which does not satisfy the Fundamental Causal Coarsening constraint on $\calJ$ is of measure zero with respect to the Lebesgue measure. To see this, first note that since $\alpha$ and $\gamma$ are fixed, the manipulation probabilities $p(j\mid \man{i})=\sum_h \alpha_{j,h,i}\gamma_h$ are fixed for each $i\in\calI,j\in\calJ$. The Fundamental Causal Coarsening constraint on $\calJ$ says ``If for some $j_1,j_2\in\calJ$ we have $p(j_1\mid\man{i})=p(j_2\mid\man{i})$ \emph{for every $i\in\calI$}, then also $p(j_1\mid i)=p(j_2\mid i)$ for every $i$.'' The subset of $P[\beta;\alpha, \gamma]$  of all distributions that \emph{do not} satisfy the constraint consists of the $P(J, H, I)$ for which for some $j_1,j_2\in\calJ$ it holds that 
\begin{equation*}
\forall_i P(j_1\mid i)=P(j_2\mid i) \textrm{ and } P(j_1\mid\man{i})\neq P(j_2\mid\man{i}).
\end{equation*}

We want to prove that this subset is measure zero. To this aim, take any pair $j_1,j_2$ and an $i$ for which $p(j_1\mid\man{i})\neq p(j_2\mid\man{i})$ (if such a configuration does not exist, then the Fundamental Causal Coarsening constraint holds for all the distributions in $P[\beta;\alpha, \gamma]$ and the proof is done). We can write 

\begin{align*}
P(j_1 \mid  i) &= \sum_h P(j_1\mid h, i) P(h\mid i)\\
         &= \frac{1}{P(i)}\sum_h P(j_1 \mid  h,i)P(i\mid h)P(h).
\end{align*}

Since the same equation applies to $P(j_2 \mid i)$, the constraint $P(j_1\mid i) = P(j_2\mid i)$ can be rewritten as

\begin{align*}
&\sum_h P(j_1 \mid  h,i)P(i\mid h)P(h)\\
  &= \sum_h P(j_2 \mid  h,i)P(i\mid h)P(h)\\
\end{align*}
which we can rewrite in terms of the independent parameters as

\begin{equation}
\sum_{h}[\alpha_{j_1,h,i}-\alpha_{j_2,h,i}]\beta_{h,i}\gamma_{h}=0,\label{eq:constraint}
\end{equation}

which is a polynomial constraint on $P[\beta;\alpha, \gamma]$. By a simple algebraic lemma \citep[proven by][]{Okamoto1973}, if the above constraint is not trivial (that is, if there exists $\beta$ for which the constraint does not hold), the subset of $P[\beta;\alpha, \gamma]$ on which it holds is measure zero.\\

(2ii) To see that Eq.~\eqref{eq:constraint} does not always hold, note that if for \emph{any} $h^*$ we set $\beta_{h^*,i}=1$ (and thus $\beta_{h,i}=0$ for any $h\neq h^*$), the equation reduces to 
\begin{align*}
(\alpha_{j_1,h^*,i}-\alpha_{j_2,h^*,i})\gamma_{h^*} = 0.
\end{align*}
Thus if Eq.~\eqref{eq:constraint} was always true, we would have $\alpha_{j_1,h,i}=\alpha_{j_2,h,i}$ or $\gamma_h=0$ for all $h$. However, this directly implies that $p(j_1\mid\man{i})\neq p(j_2\mid\man{i})$, which is a contradiction (the latter expression is false by assumption).

We have now shown that the subset of $P[\beta;\alpha, \gamma]$ which consists of distributions for which $P(j_1\mid i)=P(j_2\mid i)$ --even though $p(j_1\mid\man{i}) \neq p(j_2\mid\man{i}$ for some $i$-- is Lebesgue measure zero. Since there are only finitely many pairs of images $j_1,j_2$ for which the latter condition holds, the subset of $P[\beta;\alpha, \gamma]$ of distributions which violate the Causal Coarsening constraint is also Lebesgue measure zero (a finite sum of measure zero sets is measure zero). The remainder of the proof is a direct application of Fubini's theorem.

For each $\alpha, \gamma$, call the (measure zero) subset of $P[\beta;\alpha, \gamma]$ that violates the Causal Coarsening constraint $z[\alpha, \gamma]$. Let $Z=\cup_{\alpha, \gamma} z[\alpha, \gamma] \subset P[\alpha, \beta, \gamma]$ be the set of all the joint distributions which violate the Causal Coarsening constraint. We want to prove that $\mu(Z)=0$, where $\mu$ is the Lebesgue measure. To show this, we will use the indicator function 

\begin{equation*}
\hat{z}(\alpha, \beta, \gamma) = \left\{
  \begin{array}{l} 1\quad\text{if } \beta \in z[\alpha,\gamma],\\
                    0\quad\text{otherwise}.
  \end{array}
  \right.
\end{equation*}

By the basic properties of positive measures we have 

\begin{equation*}
\mu(Z) = \int_{P[\alpha, \beta, \gamma]}\hat{z}\;d\mu.
\end{equation*}

It is a standard application of Fubini's Theorem for the Lebesgue integral to show that the integral in question equals zero. For simplicity of notation, let
\begin{align*}
\mathcal{A} &= \mathbb{R}^{(N-1)\times K\times M}\\
\mathcal{B} &= \mathbb{R}^{(M-1)\times K}\\
\mathcal{G} &= \mathbb{R}^{K-1}.
\end{align*}

We have 

\begin{align}
\int_{P[\alpha, \beta, \gamma]}\hat{z}\;d\mu 
  &= \int_{\mathcal{A}\times\mathcal{B}\times\mathcal{G}} \hat{z}(\alpha, \beta, \gamma)\, d(\alpha, \beta, \gamma)\notag\\
  &= \int_{\mathcal{A}\times\mathcal{G}}\; \int_\mathcal{B}\hat{z}(\alpha, \beta, \gamma)\, d(\beta)\; d(\alpha, \gamma)\notag \\
  &= \int_{\mathcal{A}\times\mathcal{G}}\; \mu(z[\alpha, \gamma])\; d(\alpha, \gamma) \label{eq:measureZero:indicator}\\
  &= \int_{\mathcal{A}\times\mathcal{G}} 0\, d(\alpha, \gamma)\notag\\
  &= 0\notag.
\end{align}
Equation~\eqref{eq:measureZero:indicator} follows as $\hat{z}$ restricted to $P[\beta;\alpha, \gamma]$ is the indicator function of $z[\alpha, \gamma]$. 

This completes the proof that $Z$, the set of joint distributions over $J, H$ and $I$ that violate the Causal Coarsening constraint, is measure zero.

(2iii)
Finally, we show that (2i) and (2ii) apply if we fix an observational partition on $\calJ$ \emph{a priori}. Fixing the observational partition means fixing a set of observational constraints (OCs)

\begin{align*}
\forall_i p(j^1_1 \mid i)&=\cdots=p(j^1_{N_1}\mid i)=p^1,\\
&\vdots\\
\forall_ip(j^L_1 \mid  i)&=\cdots=p(j^L_{N_L}\mid i)=p^L,
\end{align*}

where $1\leq L\leq N$ is the number of observational classes of $\calJ$ and $N_l$ is the cardinality of the $l$th observational class (so that $N=\sum_l N_l$), and $p^1,\cdots, p^L$ are the numerical values of the observational constraints.

Since $P(J,H,I) = P(H \mid  J,I)P(J\mid I)P(I)$, $P(j\mid i)$ is an independent parameter in the unrestricted $P(J,H,I)$, and the OCs reduce the number of independent parameters of the joint by $M\sum_{l=1}^L (N_l-1)$. We want to express this parameter-space reduction in terms of the $\alpha,\beta$ and $\gamma$ parameterization from (2i) and (2ii). To do this, note first that we can write, for any $j^l_n$,
\[
\sum_hp(j^l_n, h, i) =p^l \sum_hp(h,i).
\]
Now, pick any $h^*$ for which $p(h^*,i)\neq 0$. Then we can write
\[
p(j^l_n, h^*, i) = p^l \sum_hp(h,i)-\sum_{h\neq h^*}p(j^l_n, h, i).
\]
In terms of the $\alpha, \beta, \gamma$ parameterization, this equation becomes
\[
\alpha_{j^l_n, h^*,i}\beta_{h^*,i}\gamma_{h} = p^l \sum_h \beta_{h,i}\gamma_{h}-\sum_{h\neq h^*}\alpha_{j^l_n,h,i}\beta_{h,i}\gamma_{h}
\]
or
\begin{equation}
\alpha_{j^l_n, h^*,i} = \frac{p^l \sum_h \beta_{h,i}\gamma_{h}-\sum_{h\neq h^*}\alpha_{j^l_n,h,i}\beta_{h,i}\gamma{h}}{\beta_{h^*,i}\gamma_{h}}.\label{eq:final_constraint}
\end{equation}
This full set of the OCs is equivalent to the full set of equations of this form, one for each possible $(j^l_n, i)$ combination (to the total of $M\times (N-L)$ equations as expected). Thus, we can express the range of $P(J,H,I)$ distributions consistent with a given observational partition $\Pi_o(\calJ)$ in terms of the full range of $\beta, \gamma$ parameters and a restricted number of independent $\alpha$ parameters.
\end{proof}

\subsubsection*{References}
\renewcommand{\section}[2]{}%
\bibliographystyle{plainnat}
\bibliography{bibliography}
\end{document}